\documentclass{article}

    \PassOptionsToPackage{numbers, compress}{natbib}

    \usepackage[final]{neurips_2021}

\usepackage[utf8]{inputenc} 
\usepackage[T1]{fontenc}    
\usepackage{hyperref}       
\usepackage{url}            
\usepackage{booktabs}       
\usepackage{amsfonts}       
\usepackage{nicefrac}       
\usepackage{microtype}      
\usepackage{xcolor}         
\usepackage[ruled,vlined]{algorithm2e}
\usepackage{amsmath}
\usepackage{amsthm}
\usepackage{graphicx}
\graphicspath{./figs/}
\usepackage{multirow}
\usepackage{xcolor}

\PassOptionsToPackage{sort&compress,numbers}{natbib}
\usepackage{neurips_2021}
\SetArgSty{textup}
\newcommand{\bm}[1]{\mbox{\boldmath{$#1$}}}

\theoremstyle{definition}
\newtheorem{definition}{Definition}[section]
\newtheorem{theorem}{Theorem}[section]

\DeclareMathOperator*{\argminA}{arg\,min}

\title{CBP: Backpropagation with constraint on weight precision using pseudo-Lagrange multiplier method}

\author{%
  Guhyun~Kim, Doo~Seok~Jeong\thanks{Corresponding author}\\
  Division of Materials Science and Engineering\\
  Hanyang University, Republic of Korea\\
  \texttt{guhyunkim01@gmail.com, dooseokj@hanyang.ac.kr}
}

\begin{document}

\maketitle

\begin{abstract}
  Backward propagation of errors (backpropagation) is a method to minimize objective functions (e.g., loss functions) of deep neural networks by identifying optimal sets of weights and biases. Imposing constraints on weight precision is often required to alleviate prohibitive workloads on hardware. Despite the remarkable success of backpropagation, the algorithm itself is not capable of considering such constraints unless additional algorithms are applied simultaneously. To address this issue, we propose the constrained backpropagation (CBP) algorithm based on the pseudo-Lagrange multiplier method to obtain the optimal set of weights that satisfy a given set of constraints. The defining characteristic of the proposed CBP algorithm is the utilization of a Lagrangian function (loss function plus constraint function) as its objective function. We considered various types of constraints — binary, ternary, one-bit shift, and two-bit shift weight constraints. As a post-training method, CBP applied to AlexNet, ResNet-18, ResNet-50, and GoogLeNet on ImageNet, which were pre-trained using the conventional backpropagation. For most cases, the proposed algorithm outperforms the state-of-the-art methods on ImageNet, e.g., 66.6\%, 74.4\%, and 64.0\% top-1 accuracy for ResNet-18, ResNet-50, and GoogLeNet with binary weights, respectively. This highlights CBP as a learning algorithm to address diverse constraints with the minimal performance loss by employing appropriate constraint functions. The code for CBP is publicly available at \url{https://github.com/dooseokjeong/CBP}. 
\end{abstract}

\section{Introduction}
Currently, deep learning-based methods are applied in a variety of tasks, including the classification of static data, e.g., image recognition~\cite{Taigman2014DeepFace, Krizhevsky2012ImageNet}; classification of dynamic data, e.g., speech recognition~\cite{Hinton2012Deep, Sainath2013Deep, Dahl2012Context, Hochreiter1997Long}; function approximations, which require the output of precise predictions, e.g., electronic structure predictions~\cite{Lee2019SIMPLE} and nonlinear circuit predictions~\cite{Kim2019Artificial}. All of the aforementioned tasks require discriminative models. Additionally, generative models, including generative adversarial networks~\cite{Goodfellow2014Generative} and variants~\cite{Radford2015Unsupervised, Metz2016Unrolled, Chen2016InfoGAN, Arjovsky2017Wasserstein}, comprise another type of deep neural network. Despite the diversity in application domain and model type used, almost all deep learning-based methods use backpropagation as a common learning algorithm.

Recent developments in deep learning have primarily focused on increasing the size and depth of deep neural networks (DNNs) to improve their learning capabilities as in the case of state-of-the-art DNNs like VGGNet~\cite{Simonyan2014Very} and ResNet~\cite{He2016Deep}. Given that the memory capacity required by a DNN is proportional to the number of parameters (weights and biases), memory usage for DNN becomes severe. Additionally, a significant number of multiply-accumulate operations during the training and inference stages impose prohibitive workload on hardware. Thus, efficient hardware-resource consumption is critical to the optimal performance of deep learning. One way to address this requirement is the use of weights of limited precision, such as binary~\cite{Courbariaux2015BinaryConnect,Rastegari2016XNOR} and ternary weights~\cite{Lin2015Neural,Li2016Ternary}. To this end, particular constraints are applied to weights during training, and additional algorithms for weight quantization are used in conjunction with backpropagation. This is because such constraints are not considered during the minimization of the objective function (loss function) when backpropagation is executed.

We adopt the Lagrange multiplier method (LMM) to combine basic backpropagation with additional constraint algorithms and produce a single constrained backpropagation (CBP) algorithm. We refer to the adopted method as pseudo-LMM because the constraint functions $cs\left(\bm{x}\right)$ are nondifferentiable at $\bm{x}_m \left(=\argminA_{\bm{x}} cs\left(\bm{x}\right)\right)$, rendering LMM inapplicable. Nevertheless, pseudo-LMM successfully attains the optimal point under particular conditions as for LMM. 
In the CBP algorithm, the optimal weights satisfying a given set of constraints are evaluated via a basic backpropagation algorithm. It is implemented by simply replacing the conventional objective function (loss function) with a Lagrangian function $\mathcal{L}$ that comprises the loss and constraint functions as sub-functions that are subjected to simultaneous minimization. Therefore, this method is perfectly compatible with conventional deep learning frameworks. The primary contributions of this study are as follows.

\begin{itemize}

\item We introduce a novel and simple method to incorporate given constraints into backpropagation by using a Lagrangian function as the objective function. The proposed method is able to address any set of constraints on the weights insomuch as the constraint functions are mathematically well-defined.

\item We introduce pseudo-LMM with constraint functions $cs\left(w\right)$ that are nondifferentiable at $w_m \left(=\argminA_{w}cs\left(w\right)\right)$ and analyze the kinetics of pseudo-LMM in the continuous time domain.

\item We introduce optimal (sawtooth-like) constraint functions with gradually vanishing unconstrained-weight windows and provide a guiding principle for the stable co-optimization of weights and Lagrange multipliers in a quasi-static fashion.

\item We evaluate the performance of CBP applied to AlexNet, ResNet-18, ResNet-50, and GoogLeNet (pre-trained using backpropagation with full-precision weights) with four different constraints (binary, ternary, one-bit shift, and two-bit shift weight constraints) on ImageNet as proof-of-concept examples. The results highlight the classification accuracy outperforming the previous state-of-the-art results.
\end{itemize}

\section{Related work}
The simplest approach to weight quantization is the quantization of pre-trained weights. Gong et al.~\cite{Gong2014Compressing} proposed several methods for weight quantization and demonstrated that binarizing weights using a sign function degraded the top-1 accuracy on ImageNet by less than $10\%$. Mellempudi et al.~\cite{Mellempudi2017Ternary} proposed a fine-grained quantization algorithm that calculates the optimal thresholds for the ternarization of pre-trained weights. The expectation backpropagation algorithm~\cite{Soudry2014Expectation} implements a variational Bayesian approach to weight quantization. It uses binary weights and activations during the inference stage. Zhou et al.~\cite{zhou2017incremental} proposed the incremental network quantization method (INQ) that iteratively re-trains a group of weights to compensate for the performance loss caused by the rest of weights which are quantized using pre-set quantization thresholds.

Several methods of weight quantization utilize auxiliary real-valued weights in conjunction with quantized weights during training. The straight-through-estimator (STE) comprises the conduction of forward and backpropagation using quantized weights but relies on the auxiliary real-valued weights for the update of weights~\cite{Hinton2012Deep}. BinaryConnect~\cite{Courbariaux2015BinaryConnect} utilizes weights binarized by a sign function for forward and backpropagation, and the real-valued weights are updated via backpropagation with binary weights. The binary-weight-network (BWN)~\cite{Rastegari2016XNOR} identifies the binary weights closest to the real-valued weights using a scaling factor, and it exhibits a higher classification accuracy than BinaryConnect on ImageNet. The binarized neural nets~\cite{Courbariaux2016Binarized} and XNOR-Nets~\cite{Rastegari2016XNOR} are extensions of BinaryConnect and BWN, respectively, which utilize binary activations alongside binary weights. 

Lin et al.~\cite{Lin2015Neural} proposed TernaryConnect and Ternary-weight-network (TWN), which are similar to BinaryConnect and BWN but use weight-ternarization methods instead. Trained ternary quantization (TTQ) ~\cite{zhu2016trained} also uses ternary weights that are quantized using trainable thresholds for quantization. LQ-Nets proposed by Zhang et al.~\cite{zhang2018lq} utilize activation- and weight-quantizers considering the actual distributions of full-precision activation and weight, respectively. DeepShift Network~\cite{Elhoushi_2021_CVPR} includes the LinearShift and ConvShift operators that replace the multiplications of fully-connected layers and convolution layers, respectively. Qin et al.~\cite{qin2020forward} introduced IR-Nets that feature the use of error decay estimators to approximate sign functions for weight and activation binarization to differentiable forms. Gong et al.~\cite{Gong_2019_ICCV} and Yang et al.~\cite{Yang_2019_CVPR} employed functions similar to step functions but differentiable. Pouransari et al.~\cite{Pouransari_2020_CVPR_Workshops} proposed the least squares quantization method that searches proper scale factors for binary quantization. Elthakeb et al.~\cite{elthakeb2020waveq} attempted to quantize weights and activations by applying sinusoidal regularization. The regularization involves two hyperparameters that determine weight-quantization and bitwidth-regularization strengths.   

To take into account the constraint on weight precision, Leng et al.~\cite{Leng2018Extremely} used an augmented Lagrangian function as an objective function, which includes a constraint-agreement sub-function. The method was successfully applied to various DNNs on ImageNet; yet, the method failed to reach the accuracy level for full-precision models even when 3-bit precision was used.

The CBP algorithm proposed in this study also utilizes LMM; but CBP essentially differs from \cite{Leng2018Extremely} given that (i) the basic differential multiplier method (BDMM), rather than ADMM, is used to apply various constraints on weight precision, (ii) particularly designed constraint functions with gradually vanishing unconstrained-weight windows are used, and (iii) substantial improvement on the classification accuracy on ImageNet is achieved.

\section{Optimization method}
\subsection{Pseudo-Lagrange multiplier method}
LMM calculates the maximum or minimum value of a differentiable function $f$ under a differentiable constraint $cs=0$.~\cite{Bertsekas2014Constrained} Let us assume that a function $f\left(x,y\right)$ attains its minimum (or maximum) value $m$ satisfying a given constraint at $\left(x_m, y_m\right)$, i.e., $f\left(x_m, y_m\right)=m$. Further, $cs\left(x_m,y_m\right)=0$ as the constraint is satisfied at this point. In this case, the point of intersection between the graphs of the two functions, $f\left(x,y\right)=m$ and $cs\left(x,y\right)=0$, is $\left(x_m,y_m\right)$. Because the two functions have a common tangent at the point of intersection, the following equation holds:
\begin{equation}
    \nabla_{x,y}f=-\lambda \nabla_{x,y}cs,
    \label{lag1}
\end{equation}
at $\left(x_m,y_m\right)$, where $\lambda$ is a Lagrange multiplier~\cite{Luenberger2015Linear}.

The Lagrangian function is defined as  $\mathcal{L}\left(x,y,\lambda\right)=f\left(x,y\right)+\lambda cs\left(x,y\right)$. Consider a local point $\left(x,y\right)$ at which the gradient of $\mathcal{L}\left(x,y,\lambda\right)$ is zero. Therefore, $\nabla_{x,y,\lambda}\mathcal{L}\left(x,y,\lambda\right)=\nabla_{x,y}\left[f\left(x,y\right)+\lambda cs\left(x,y\right)\right]+\nabla_{\lambda}\left[\lambda cs\left(x,y\right)\right]$. Thus, $\nabla_{x,y,\lambda}\mathcal{L}\left(x,y,\lambda\right)=\bm{0}$ is equivalent to the following equations.
\begin{eqnarray}\label{lmm1}
\nabla_{x,y,\lambda}\mathcal{L}\left(x,y,\lambda\right)=\bm{0} & \Leftrightarrow & \left\{
\begin{array}{rl}
\nabla_{x,y}\left[f\left(x,y\right)+\lambda cs\left(x,y\right)\right] & = \bm{0},\\
\nabla_{\lambda}\left[\lambda cs\left(x,y\right)\right] & =  0 \\
\end{array}\right.\nonumber\\
& \Leftrightarrow & \left\{
\begin{array}{rl}
    \nabla_{x,y}f\left(x,y\right) & = -\lambda \nabla_{x,y} cs\left(x,y\right)\\
    cs\left(x,y\right) & = 0 
\end{array}\right.
\end{eqnarray}
This satisfies the condition in Eq. \eqref{lag1} as well as the constraint $cs\left(x,y\right)=0$. Therefore, the local point $\left(x,y\right)$ corresponds to the minimum (or maximum) point of the function $f$ under the constraint.

We define pseudo-LMM to address similar minimization tasks but with continuous constraint functions that are nondifferentiable at $\left(x_m, y_m\right) =\argminA_{x,y}cs\left(x,y\right)$. Thus, Eq. \eqref{lag1} cannot be satisfied at the optimal point. Nevertheless, pseudo-LMM enables us to \textit{minimize} the function $f\left(\bm{x}\right)$ subject to the constraint condition $cs\left(\bm{x}\right) = 0$ using the Lagrangian function $\mathcal{L}\left(\bm{x},\lambda\right)$. 
\theoremstyle{definition}
\begin{definition}[Pseudo-LMM]
Pseudo-LMM is a method to attain the optimal variables $\bm{x}_m$ that minimize function $f\left(\bm{x}\right)$ subject to the constraint condition $cs\left(\bm{x}\right) = 0$, where the function $cs\left(\bm{x}\right)$ is nondifferentiable at $\bm{x}_m$ but reaches the minimum at $\bm{x}_m$, i.e., $cs\left(\bm{x}_m\right)=0$.
\end{definition}
\begin{theorem}\label{theorem1}
Minimizing the Lagrangian function $\mathcal{L}\left(\bm{x},\lambda\right)$, which is given by $\mathcal{L}\left(\bm{x},\lambda\right) = f\left(\bm{x}\right) + \lambda cs\left(\bm{x}\right)$, is equivalent to minimizing the function $f\left(\bm{x}\right)$ subject to $cs\left(\bm{x}\right) = 0$
\end{theorem}
\begin{proof}
The Lagrangian function $\mathcal{L}\left(\bm{x},\lambda\right)$ is always differentiable with respect to the Lagrangian multiplier $\lambda$, so that the equation $cs\left(\bm{x}\right) = 0$ holds at the optimal point $\bm{x}_m$. Thus, we have
\begin{equation}\label{plmm1}
\underset{\bm{x}, \lambda}{\text{minimize }} \mathcal{L}\left(\bm{x},\lambda\right) \Leftrightarrow \left\{
\begin{aligned}
& \underset{\bm{x}}{\text{minimize}}
& & \mathcal{L}(\bm{x};\lambda) \\
& \text{subject to}
& & cs\left(\bm{x}\right) = 0.
\end{aligned}\right.\\
\end{equation}
When the constraint is satisfied, i.e., $cs\left(\bm{x}\right)=0$, the Lagrangian function $\mathcal{L}\left(\bm{x};\lambda\right)$ equals the function $f\left(\bm{x}\right)$, so that the task to minimize $\mathcal{L}\left(\bm{x},\lambda\right)$ with respect to $\bm{x}$ and $\lambda$ corresponds to the task to minimize $f\left(\bm{x}\right)$ subject to $cs\left(\bm{x}\right)=0$.
\end{proof}
Given Theorem {\ref{theorem1}}, pseudo-LMM can attain the optimal point by minimizing the Lagrangian function $\mathcal{L}$ in spite of the nondifferentiability of the constraint function $cs\left(\bm{x}\right)=0$ at the optimal point $\bm{x}_m$. Note that not all functions have zero gradients at their minimum points; for instance, the function $y=\left|x\right|$ attains its minimum at $x=0$ but the gradient at the minimum point is not defined. However, all convex functions have zero gradients at their minimum points. Thus, LMM to minimize the function $f\left(\bm{x}\right)$ subject to the constraint convex function $cs\left(\bm{x}\right)$ with minimum point $\bm{x}_m$ is a subset of pseudo-LMM. In this case, the constraint function has zero gradient at the minimum point, so that Eq. \eqref{lmm1} becomes
\begin{eqnarray}\label{lmm2}
\nabla_{\bm{x},\lambda}\mathcal{L}\left(\bm{x},\lambda\right)=\bm{0} & \Leftrightarrow & \left\{
\begin{array}{rl}
    \nabla_{\bm{x}}f\left(\bm{x}\right) & = 0  \\
    \text{subject to }cs\left(\bm{x}\right) & = 0. 
\end{array}\right.
\end{eqnarray}

We will consider \textit{continuous} constraint functions with a few nondifferentiable points in their variable domains, including their minimum points. Other than such nondifferentiable points, we will use the gradient descent method to search for the minimum points within the framework of pseudo-LMM. Hereafter, when the gradient of the Lagrangian $\mathcal{L}$ function is remarked, its variable domain excludes such nondifferentiable points.

The optimal solution to Eq. \eqref{lmm2} can be found using the basic differential multiplier method (BDMM)~\cite{Platt1987Constrained} that calculates the point at which $\mathcal{L}\left(\bm{x},\bm{\lambda}\right)$ attains its minimum value by driving $\bm{x}$ toward the constraint subspace $\bar{x}$ $\left(\bm{cs}\left(\bar{x}\right)=0\right)$. The BDMM updates $\bm{x}$ and $\bm{\lambda}$ according to the following relations.
\begin{equation}
    \bm{x} \leftarrow \bm{x} - \eta_x\nabla_{\bm{x}}\mathcal{L}\left(\bm{x},\bm{\lambda}\right),
    \label{bdmm1}
\end{equation}
and
\begin{equation}
    \bm{\lambda} \leftarrow \bm{\lambda} + \eta_{\lambda}\nabla_{\bm{\lambda}}\mathcal{L}\left(\bm{x},\bm{\lambda}\right).
    \label{bdmm2}
\end{equation}
BDMM is cheaper than Newton's method in terms of computational cost. Additionally, Eq. \eqref{bdmm1} is identical to the solution used in the gradient descent method during backpropagation, except for the use of a Lagrangian function instead of a loss function. This indicates the compatibility of pseudo-LMM with the optimization framework based on backpropagation.

\subsection{Constrained backpropagation using the pseudo-Lagrange multiplier method}
We utilize pseudo-LMM to train DNNs with particular sets of weight-constraints. We define a Lagrangian function $\mathcal{L}$ in the context of feedforward DNN using the following relation.
\begin{eqnarray}
    \mathcal{L}\left(\bm{y}^{(k)}, \hat{\bm{y}}^{(k)};\bm{W},\bm{\lambda}\right)&=&C\left(\bm{y}^{(k)}, \hat{\bm{y}}^{(k)};\bm{W}\right)+\bm{\lambda}^{\rm T}\bm{cs}\left(\bm{W}\right),\\\nonumber
    \bm{cs}&=&\left[cs_1\left(w_1\right),\ldots, cs_{n_w}\left(w_{n_w}\right)\right]^{\rm T}\\\nonumber
    \bm{\lambda}&=&\left[\lambda_1,\ldots, \lambda_{n_w}\right]^{\rm T}.
    \label{lf}
\end{eqnarray}
where $\bm{y}^{(k)}$ and $\hat{\bm{y}}^{(k)}$ denote the actual output vector for the $k$th input data and correct label, respectively. The set $\bm{W}$ denotes a set of weight matrices, including $n_w$ weights in aggregate, and the function $C$ denotes a loss function. 
Each weight $w_i$ $\left(1 \leq i \leq n_w \right)$ is given one constraint function $cs_i\left(w_i\right)$ and one multiplier $\lambda_i$.

We chose sawtooth-shaped constraint functions. We quantize the real-valued weights into $n_q$ values in the set $\bm{Q}=\left\{q_i\right\}_{i=1}^{n_q}$, where $q_i < q_{i+1}$ for all $i$. We also employ a set of the medians of neighboring values in the set $\bm{Q}$: $\bm{M}=\left\{m_i\right\}_{i=1}^{n_q-1}$, where $m_i = \left(q_i+q_{i+1}\right)/2$. Using $\bm{Q}$ and $\bm{M}$, we define a partial constraint function $y_i$ for $i=0$, $1 \leq i < n_q$, and $i=n_q$,    
\begin{eqnarray}\label{cs_y}
    y_0\left(w\right)&=&\left\{
    \begin{array}{cl}
        -2\left(w-q_1\right) & \text{if} \quad w < q_1,\\
        0 & \quad \text{otherwise},
    \end{array}\right.\nonumber\\ 
    y_i\left(w\right)&=&\left\{
    \begin{array}{cl}
        -2\left|w-m_i\right|+q_{i+1}-q_i & \text{if} \quad q_i \leq w < q_{i+1},\\
        0 & \text{otherwise},
    \end{array}\right.\nonumber\\
    y_{n_q}\left(w\right)&=&\left\{
    \begin{array}{cl}
        2\left(w-q_{n_q}\right) & \text{if} \quad w \geq q_{n_q},\\
        0 & \quad \text{otherwise},
    \end{array}\right. 
\end{eqnarray}
respectively. The constraint function $cs$ is the summation of the partial constraint functions $Y\left(w\right)=\sum_{i=0}^{n_q}y_i\left(w\right)$, gated by the unconstrained-weight window $ucs\left(w\right)$ parameterized by a variable $g$.   
\begin{equation}\label{csf}
    cs\left(w; \bm{Q},\bm{M}, g\right)=ucs\left(w\right)Y\left(w\right),
\end{equation}
where
\begin{equation}\label{ucs1}
    ucs\left(w\right)=1-\sum_{i=0}^{n_q-1}H\left(\dfrac{1}{2g}\left(q_{i+1}-q_i\right)-\left|w-m_i+\epsilon\right|\right), 
\end{equation}
where $\epsilon\rightarrow0^+$, and $H$ denotes the Heaviside step function. The function $ucs\left(w\right)$ realizes the unconstrained-weight window as a function of $g \left(\geq 1\right)$. When $g = 1$, the function outputs zero for $q_1 \leq w < q_{n_q}$, merely confining $w$ to the range $q_1 \leq w < q_{n_q}$ without weight quantization, whereas, when $g\rightarrow\infty$, the window vanishes, allowing the constraint function to quantize the weight in the entire weight range. Examples of function $ucs\left(w\right)$ are shown in Fig.~\ref{fig1}. 
\begin{figure}\centering
    \includegraphics[width=6in]{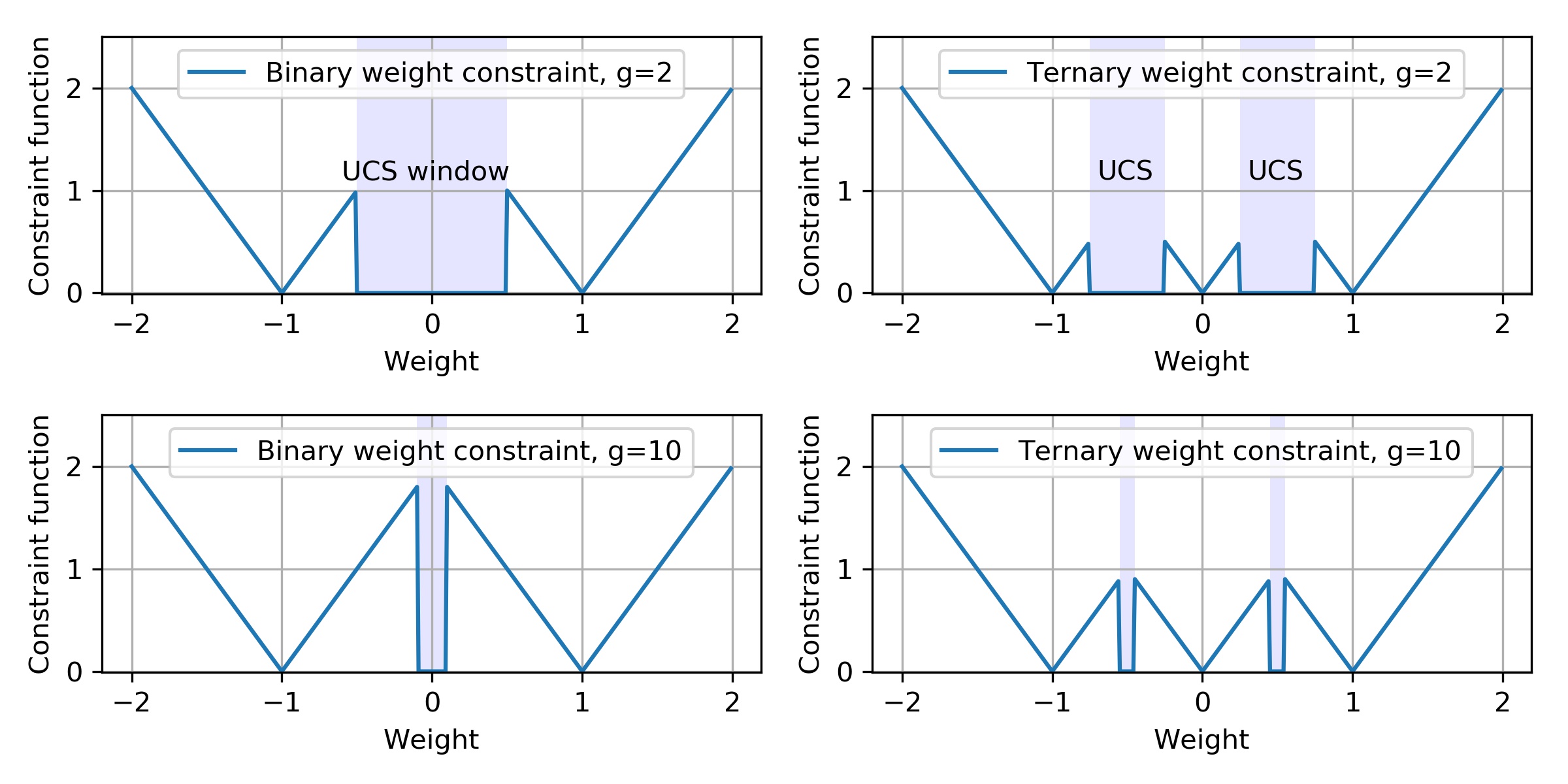}
    \caption{\label{fig1} Binary- and ternary-weight constraint functions for $g$ = 2 and 10. Blue-filled regions indicate unconstrained-weight windows.}
\end{figure}

The unconstrained-weight window variable $g$ is initially set to one and updated such that it keeps increasing during training, i.e., the window gradually vanishes. The window gradually vanishing allows sequential weight quantization such that the further the initial weights from their nearest $q_i$, the later their weights are subject to quantization, which is otherwise subject to simultaneous (abrupt) quantization. It is likely that the further the initial weights from their nearest $q_i$, the larger the increase in loss function $C$ when they are quantized. Thus, the sequential quantization from the weights close to their $q_i$ likely avoids an abrupt increase in the loss. Further, while the closer weights are being quantized, the further weights (not subject to quantization yet) are being updated to reduce the loss given the partially quantized weights. This effect will be discussed in Sections \ref{sec:kinetics} and \ref{discussion}. 

For every training batch, the weights are updated following a method similar to conventional backpropagation. Nevertheless, the use of the Lagrangian function in Eq.~\eqref{lf}, rather than a loss function only, as an objective function constitutes a critical difference. The Lagrange multipliers $\bm{\lambda}$ are subsequently updated using the gradient ascent method in Eq.~\eqref{bdmm2}. Updating cross-coupled variables, such as $\bm{W}$ and $\bm{\lambda}$, often experiences difficulties in convergence toward the optimal values because of oscillation around the optimal values. A feasible solution involves quasi-static update. To this end, we significantly reduce the update frequency of the Lagrange multipliers $\bm{\lambda}$ compared with weights $\bm{W}$. 

\textbf{Weight update:} Weights $\bm{W}$ are updated once every iteration as for the conventional backpropagation but using the Lagrangian function $\mathcal{L}$.

\textbf{Lagrange multiplier update:} Lagrange multipliers $\bm{\lambda}$ are conditionally updated once every training epoch. The update is allowed if the summation of all $\mathcal{L}$ in a given epoch $\left(\mathcal{L}_{sum}\right)$ is not smaller than $\mathcal{L}_{sum}$ for the previous epoch $\left(\mathcal{L}_{sum}^{pre}\right)$ or the multipliers $\bm{\lambda}$ have not been updated in the past $p_{max}$ epochs. This achieves the convergence of $\bm{W}$ for a given $\bm{\lambda}$ in a quasi-static manner. 

\textbf{Unconstrained weight window update:} Unconstrained-weight window variable $\bm{g}$ is updated on the same condition as for the Lagrange multipliers $\bm{\lambda}$. Unlike weights $\bm{W}$ and multipliers $\bm{\lambda}$, the variable $\bm{g}$ (initialized to one) constantly increases when updated such that $\Delta g = 1$ when $g < 10$, $\Delta g = 10$ for $10\leq g<100$, and $\Delta g = 100$ otherwise.   

The detailed learning algorithm is shown in the pseudocode in Appendix~\ref{app:pseudocode}.

\subsection{Learning kinetics}\label{sec:kinetics}
Learning with BDMM using the Lagrangian function is better understood in the continuous time domain. We first address the kinetics of learning without the unconstrained-weight window $ucs\left(w\right)$. The change in the Lagrangian function $\mathcal{L}$ at a given learning step in the discrete time domain is equivalent to the derivative of $\mathcal{L}$ with time at time $t$ in the continuous time domain, which is given by
\begin{equation}\label{lyapunov3}
    \dfrac{d\mathcal{L}}{dt} = -\tau_{W}^{-1}\sum_{i=0}^{n_w}\left(\dfrac{\partial C}{\partial w_i}+\lambda_i\dfrac{\partial cs_i}{\partial w_i}\right)^2 + \tau_{\lambda}^{-1}\sum_{i=0}^{n_w}cs_i^2.
\end{equation}
The Lagrange multiplier $\lambda_i$ at time $t$ is given by
\begin{equation}\label{multiplier}
    \lambda_i\left(t\right)=\lambda_i\left(0\right)+\tau_{\lambda}^{-1}\int_0^t cs_i dt.
\end{equation}
Eqs.~\eqref{lyapunov3} and \eqref{multiplier} are derived in Appendix~\ref{kinetics_general}. 
The constraint functions $cs_i$ approach zero as the weights approach their corresponding quantized values $q_i$, and thus the Lagrange multipliers in Eq. \eqref{multiplier} asymptotically converge to their limits.

At equilibrium, the Lagrange function is no longer time-dependent, i.e., $d\mathcal{L}/dt=0$. This requires the Lagrange multipliers reaching their limits, which in turn requires the weights reaching their corresponding quantized values $\bm{W}^{*cs}$, leading to $\bm{cs}=\bm{0}$. For convenience, we define the integration of $cs_i$ in Eq. \eqref{multiplier} as $\Delta cs_i\left(\geq 0\right)$. 
\begin{equation}\label{dcs}
    \Delta cs_i=\int_0^t cs_i dt, \quad \text{if }cs_i\rightarrow 0\text{ as } t\rightarrow\infty.
\end{equation}
Thus, the equilibrium Lagrange multiplier $\lambda_i^{*cs}$ can be expressed as
\begin{equation}\label{multiplier_equil}
    \lambda_i^{*cs}=\lambda_i\left(0\right)+\tau_{\lambda}^{-1}\Delta cs_i.
\end{equation}
Therefore, it is evident from Eq. \eqref{lyapunov3} that the equilibrium leads to 
\begin{equation}
    \forall i, \dfrac{\partial C}{\partial w_i}=-\left(\lambda_i\left(0\right)+\tau_{\lambda}^{-1}\Delta cs_i\right)\dfrac{\partial cs_i}{\partial w_i}.
\end{equation}
We consider sawtooth constraint functions with slopes $\pm s$, i.e., $\partial cs_i/\partial w_i=\pm s$, where $s>0$. Eq.~\eqref{cs_y} is the case of $s=2$. Generally, the Lagrange multiplier is initialized to zero, i.e., $\lambda_i\left(0\right)=0$. Therefore, the gradient of loss function $C$ at the equilibrium point $\bm{W}^{*cs}$ is given by
\begin{equation}\label{loss_equil}
    \forall i, \dfrac{\partial C}{\partial w_i}=\pm\tau_{\lambda}^{-1}s\Delta cs_i=\pm\lambda_i^{*cs}s.
\end{equation}
Consider that the loss function $C$ has the equilibrium point $\bm{W}^*\left(=\argminA_{\bm{W}} C\right)$. Eq. \eqref{loss_equil} elucidates the increase of loss by attaining the equilibrium point $\bm{W}^{*cs}$. In this regard, $\lambda_i^{*cs}$ corresponds to the cost of weight quantization. Assuming the convexity of the loss function $C$ in a domain $D$ including $\bm{W}^{*}$ and $\bm{W}^{*cs}$, $C\left(\bm{W}^{*cs}\right)$ keeps increasing as $\lambda_i^{*cs}$ increases. If the initial pre-trained weights equal their corresponding quantized weights, i.e., $w_i^*=w_i^{*cs}$, then $\Delta cs_i=0$, and thus $\lambda_i^{*cs}=0$ according to Eq. \eqref{multiplier_equil}. Eq. \eqref{loss_equil} consequently yields $\partial C/\partial w_i=0$, indicating zero cost of quantization.

Considering the gradually vanishing unconstrained-weight window $ucs\left(w\right)$ yields the derivative of $\mathcal{L}$ with time at time $t$ in the continuous time domain as follows.
\begin{equation}\label{lyapunov_mod3}
    \dfrac{d\mathcal{L}}{dt} = -\tau_{W}^{-1}\sum_{i=0}^{n_w}\left(\dfrac{\partial C}{\partial w_i}+\lambda_iucs_i\dfrac{\partial Y_i}{\partial w_i}\right)^2 + \tau_{\lambda}^{-1}\sum_{i=0}^{n_w}\left(ucs_iY_i\right)^2.
\end{equation}
The derivation of Eq.~\eqref{lyapunov_mod3} is given in Appendix~\ref{kinetics}. 
Distinguishing the weights in the unconstrained-weight window $D_{ucs}$ from the others at a given time $t$, Eq. \eqref{lyapunov_mod3} can be written by
\begin{equation}\label{lyapunov_mod4}
    \dfrac{d\mathcal{L}}{dt} = -\tau_{W}^{-1}\sum_{w_i\in D_{ucs}}\left(\dfrac{\partial C}{\partial w_i}\right)^2 -\sum_{w_i\notin D_{ucs}}\left[\tau_{W}^{-1}\left(\dfrac{\partial C}{\partial w_i}+\lambda_i\dfrac{\partial Y_i}{\partial w_i}\right)^2 - \tau_{\lambda}^{-1}Y_i^2\right].
\end{equation}
The latter term on the right-hand side of Eq. \eqref{lyapunov_mod4} indicates that the weights outside the window $D_{ucs}$ are being quantized at the cost of increase of loss. However, as indicated by the former term, the weights in the window $D_{ucs}$ are being optimized only to decrease the loss function with partially quantized weights. Compare this gradual quantization with abrupt quantization without the gradually vanishing unconstrained-weight window, where all weights are subject to simultaneous (abrupt) quantization. The gradual quantization allows the weights in the window to further reduce the loss function regarding the weights that have already been quantized or are being quantized, and thus the eventual cost of quantization is likely smaller than the simultaneous quantization case.

\section{Experiments}
To evaluate the performance of our algorithm, we trained three models (AlexNet, and ResNet-18 and 50) on the ImageNet dataset \cite{ILSVRC15} with four different weight constraints (binary, ternary, and one-bit, and two-bit shift weight constraints). ImageNet consists of approximately 1.2 million training images and 50 thousands validation images. All training images were pre-processed such that they were randomly cropped and resized to $224\times224$ with mean subtraction and variance division. Additionally, random horizontal flipping and color jittering were applied. For validation, the images were resized to $256\times256$ and their centers in $224\times224$ were cropped. We evaluated the top-1 and top-5 classification accuracies on the validation set.

We considered binary, ternary, one-bit shift and two-bits shift weight constraints to validate the CBP algorithm as a general weight-quantization framework. For all cases, we introduced layer-wise scaling factors $a$ such that $a^{(l)}$ (for the $l$th layer) is given by $a^{(l)} = \Vert\bm{W}^{(l)}\Vert_1/n^{(l)}$, where $\bm{W}^{(l)}$ and $n^{(l)}$ denote the weight matrix of the $l$th layer and the number of elements of $\bm{W}^{(l)}$, respectively. As for \cite{Rastegari2016XNOR} and \cite{Li2016Ternary}, the weight matrices of the first and last layers were not quantized. The quantized weights employed for each constraint case is elaborated as follows.

\textbf{Binary-weight constraint:} A set of quantized weights $\bm{Q}$ is $\left\{-a,a\right\}$.

\textbf{The other weight constraints:} A set of quantized weights $\bm{Q}$ is $\left\{0, \pm 2^{-d}a\right\}_{d=0}^{D}$, where $D=0$, 1 and 2 for the ternary, one-bit shift, and two-bit shift weight constraints. Each ternary weight needs 2-bit memory while each of one-bit and two-bit shift weight needs 3-bit memory. 

We adopted the STE~\cite{bengio2013estimating} to train the models such that the forward pass is based on quantized weights $w_q$, 
\begin{equation}
    w_q = q_1+\sum_{i=1}^{n_q-1}\left(q_{i+1}-q_{i}\right)\left(sign\left(w-m_i\right)+1\right)/2,\nonumber
\end{equation}
whereas the backward pass uses the real-valued weights $w$ that are subject to quantization, $\partial\mathcal{L}/\partial w = \partial\mathcal{L}/\partial w_q$.

For all cases, the DNN was pre-trained using conventional backpropagation with full-precision weights and activations, which was followed by post-learning using CBP. We used the stochastic gradient descent with momentum to minimize the Lagrangian function $\mathcal{L}$ with respect to $\bm{W}$ and Adam~\cite{Kingma2014Adam} to maximize $\mathcal{L}$ with respect to $\bm{\lambda}$. The initial multiplier-learning rate $\eta_{\lambda}$ and $p_{max}$ were set to $10^{-4}$ and 20, respectively. The weight-learning rate $\eta_{W}$ decreased to $10^{-1}$ times the initial rate when $g$ reached 20 for all cases except GoogLeNet with the binary-weight constraint (the weight-learning rate decayed when $g=200$). The hyperparameters used are shown in Appendix~\ref{app:hyperparameters}, which were found using manual searches.

By asymptotically minimizing the Lagrangian function $\mathcal{L}$, the constraint function $\bm{cs}\left(\bm{W}\right)$ approaches $\bm{0}$. The degree of constraint-failure per weight was evaluated based on the constraint-failure score ($CFS$), which is defined as
\begin{equation}\label{cfs}
CFS=\dfrac{1}{n_w}\sum_{i=1}^{n_w}Y_i\left(w_i;\bm{Q},\bm{M}\right),
\end{equation}
where $n_w$ denotes the total number of weights. The CBP algorithm was implemented in Python on a workstation (CPU: Intel Xeon Silver 4110 2.10GHz, GPU: Titan RTX). 

It should be noted that we used CBP as a post-training method, so that the random seed effect is involved only when organizing the mini-batches. The accuracy deviation is consequently marginal. 

\begin{table}[hbt!]
  \caption{Top-1/Top-5 accuracy of AlexNet, ResNet-18, ResNet-50, and GoogLeNet on ImageNet}
  \label{comparison}
  \centering
  \resizebox{\textwidth}{!}{
  \begin{tabular}{cccccc}
    \toprule
    Algorithm & Binary & Ternary & One-bit shift & Two-bit shift & Full-precision \\
    \midrule
    \multicolumn{6}{c}{AlexNet}\\
    \midrule
    BWN~\cite{Rastegari2016XNOR} & 56.8\%/79.4\% & - & - & - &\multirow{5}{*}{60.0\%/82.4\%} \\
    ADMM~\cite{Leng2018Extremely} & 57.0\%/79.7\% & 58.2\%/80.6\% & 59.2\%/81.8\% & 60.0\%/82.2\%&\\
    LQ-Nets~\cite{zhang2018lq} & - & 60.5\%/82.7\% & - & - &\\
    TTQ~\cite{zhu2016trained} & - &57.5\%/79.7\% & - & - &\\
    \textbf{CBP} & \bm{58.0\%/80.6\%} & \bm{58.8\%/81.2\%} & \bm{60.8\%/82.6\%} & \bm{60.9\%/82.8\%} &\\
    \midrule
    \multicolumn{6}{c}{ResNet-18}\\
    \midrule
    BWN~\cite{Rastegari2016XNOR} & 60.8\%/83.0\% & - & - & - &\multirow{8}{*}{69.6\%/89.2\%} \\
    TWN~\cite{Li2016Ternary} & - & 61.8\%/84.2\% & - & -&\\
    INQ~\cite{zhou2017incremental} & - & 66.0\%/87.1\% & - & 68.1\%/88.4\%&\\
    ADMM~\cite{Leng2018Extremely} & 64.8\%/86.2\% & 67.0\%/87.5\% & 67.5\%/87.9\% & 68.1\%/88.3\%&\\
    QN~\cite{Yang_2019_CVPR} & 66.5\%/87.3\% & 69.1\%/88.9\% & 69.9\%/89.3\% & 70.4\%/89.6\% &\\
    IR-Nets~\cite{qin2020forward} & 66.5\%/86.8\% & - & - & - &\\
    LQ-Nets~\cite{zhang2018lq} & - & 68.0\%/88.0\% & - & 69.3\%/88.3\%&\\
    TTQ~\cite{zhu2016trained} & - &66.6\%/87.2\% & - & - &\\
    DSQ~\cite{Gong_2019_ICCV} & 63.71\%/- & - & - & -&\\
    LS~\cite{Pouransari_2020_CVPR_Workshops} & 66.1\%/86.5 & - & - & -&\\
    \textbf{CBP} & \bm{66.6\%/87.1\%} & \bm{69.1\%/89.0\%} & \bm{69.6\%/89.3\%} & \bm{69.6\%/89.3\%} &\\
    \midrule
    \multicolumn{6}{c}{ResNet-50}\\
    \midrule
    BWN~\cite{Rastegari2016XNOR} & 63.9\%/85.1\% & - & - & - &\multirow{4}{*}{76.0\%/93.0\%} \\
    TWN~\cite{Li2016Ternary} & - & 65.6\%/86.5\% & - & -&\\
    ADMM~\cite{Leng2018Extremely} & 68.7\%/88.6\% & 72.5\%/90.7\% & 73.9\%/91.5\% & 74.0\%/91.6\%&\\
    QN~\cite{Yang_2019_CVPR} & 72.8\%/91.3\% & 75.2\%/92.6\% & 75.5\%/92.8\% & 76.2\%/93.2\% &\\
    \textbf{CBP} & \bm{74.4\%/92.1\%} & \bm{75.1\%/92.5\%} & \bm{76.0\%/92.9\%} & \bm{76.0\%/92.9\%} &\\
    \midrule
    \multicolumn{6}{c}{GoogLeNet}\\
    \midrule
    BWN~\cite{Rastegari2016XNOR} & 59.0\%/82.4\% & - & - & - &\multirow{4}{*}{71.0\%/90.8\%} \\
    TWN~\cite{Li2016Ternary} & - & 61.2\%/86.5\% & - & -&\\
    ADMM~\cite{Leng2018Extremely} & 60.3\%/83.2\% & 63.1\%/85.4\% & 65.9\%/87.3\% & 66.3\%/87.5\%&\\
    \textbf{CBP} & \bm{64.0\%/86.0\%} & \bm{66.0\%/87.3\%} & \bm{69.8\%/89.7\%} & \bm{70.5\%/90.1\%} &\\
    \bottomrule
  \end{tabular}
  }
 \end{table}

\subsection{AlexNet}
AlexNet is a simple convolutional networks which consists of five convolutional layers and three fully-connected layers \cite{Krizhevsky2012ImageNet}. We used AlexNet with batch normalization~\cite{ioffe2015batch} as in~\cite{Rastegari2016XNOR, Li2016Ternary, Leng2018Extremely}. The initial weight-learning rate $\eta_{W}$ was set to $10^{-3}$ for the binary- and ternary-weight constraints and $10^{-4}$ for the other constraints. The batch size was set to 256. We used a weight decay rate (L2-regularization) of $5\times10^{-4}$.
The CBP algorithm exhibited state-of-the-art results as listed in Table.~\ref{comparison}. The detailed behaviors of networks with binary- and ternary-weight constraints are addressed in Appendix \ref{app:data}. The behaviors highlight asymptotic increases in the top-1 and top-5 recognition accuracy with asymptotic decrease in $CFS$. Consequently, the weight distribution bifurcates asymptotically, fulfilling the constraints imposed on the weights.


\subsection{ResNet-18 and ResNet-50}
We also evaluated our algorithm on ResNet-18 and ResNet-50~\cite{He2016Deep} which were pre-trained using conventional backpropagation. For ResNet-18, the initial weight-learning rate $\eta_{W}$ was set to $10^{-3}$ for all constraint cases. The batch size was 256. For ReNet-50, the initial weight-learning rate $\eta_{W}$ was set to $10^{-3}$ for binary- and ternary-weight constraints and $10^{-4}$ for the other cases. The batch size was set to 128. The weight decay rate (L2-regularization) was set to $10^{-4}$ for both ResNet-18 and ResNet-50. The results are summarized in Table~\ref{comparison}, highlighting state-of-the-art performance compared with previous results. Notably, CBP with the one- and two-bit shift weight constraints almost reaches the performance of the full-precision networks. Particularly, for ResNet-50, CBP with the binary-weight constraint significantly outperforms other methods. The detailed behaviors of weight quantizations for ResNet-18 and ResNet-50 are addressed in Appendix~\ref{app:data}.

\subsection{GoogLeNet}
GoogLeNet consists of 22 layers organized with the inception modules~\cite{Szegedy_2015_CVPR}. We evaluated our algorithm on GoogLeNet which was pre-trained using conventional backpropagation. The weight-learning rate $\eta_{W}$ was initially set to $10^{-3}$ for all constraint cases. The batch size, $p_{max}$, and weight decay rate were set to 256, 10 and $10^{-4}$, respectively. 
For the binary-weight constraint case, the weight-learning rate $\eta_{W}$ decreased to $10^{-1}$ times the initial rate when $g$ reached 200. The results are summarized in Table~\ref{comparison}. Notably, CBP significantly outperforms the previous results for all constraint cases. The detailed behaviors of weight quantizations for GoogLeNet are addressed in Appendix~\ref{app:data}.

\section{Discussion}\label{discussion}
To evaluate the effect of the constraint function on training performance, we considered three different cases of post-training a DNN using CBP (i) with and (ii) without the unconstrained-weight window, and (iii) without the constraint function at all, i.e., conventional backpropagation with STE only. Because all DNNs in this work include STE, the comparison between these three cases highlights the effect of the constraint function in addition to STE. For all cases, pre-training using conventional backpropagation preceded the three different post-training schemes. Table~\ref{ablation} addresses the comparison, highlighting accuracy and $CFS$ improvement in Case (i) over Case (iii). This indicates that CBP allows the DNN to learn features while the weights are being quantized by the constraint function with the gradually vanishing unconstrained-weight window. On the contrary, CBP without the unconstrained-weight window (Case (ii)) rather degraded the accuracy compared with Case (iii), whereas the improvement on $CFS$ was significant. This may be because the constraint function without unconstrained-weight window strongly forced the weights to be quantized without learning the features.

\begin{table}[hbt!]
  \caption{Top-1 accuracy of ResNet-18 trained in various conditions}
  \label{ablation}
  \centering
  \begin{tabular}{ccc}
    \toprule
    Post-training algorithm & Accuracy & $CFS$ \\
    \midrule
    CBP with update of $g$ & 66.6\%/87.1\% & 1.19 $\times 10^{-3}$\\
    CBP without unconstrained-weight window & 60.2\%/82.7\% & 1.05$\times 10^{-5}$\\
    Backpropagation+STE & 64.6\%/85.9\% & 3.58$\times 10^{-2}$\\
    \bottomrule
  \end{tabular}
\end{table}

We used manual searches for the hyperparameters, weight-learning rate $\eta_{W}$, multiplier-learning rate $\eta_{\lambda}$, multiplier update scheduling variable $p_{max}$, and unconstrained-weight window variable $\Delta g$. We used identical parameters $\eta_{\lambda}$, $p_{max}$, and $\Delta g$ for all four models, each with the four distinct constraints, i.e., 12 cases in total. CBP is unlikely susceptible to the hyperparameters for different models. Therefore, the hyperparameters used in this work may serve as the decent initial values for other models.

CBP needs floating-point operations (FLOPs) for the Lagrange multiplier update in addition to FLOPs for the weight update, which causes additional computational complexity. Given that a Lagrange multiplier is assigned to each weight, the additional complexity scales with the number of weights. The total computational complexity of CBP exceeds the conventional backpropagation by approximately 2\% for ResNet-18 and ResNet-50 whereas by approximately 25\% for AlexNet. The complexity estimation is elaborated in Appendix~\ref{app:complexity}.

We used CBP as a post-training method. That is, the networks considered were pre-trained using conventional backpropagation. Applying CBP to untrained networks hardly reached the accuracies of classification listed in Table~\ref{comparison}. When efficiency in training is of the most important concern, CBP may not be the best choice. However, when efficiency in memory usage is of the most important concern, CBP may be the optimal choice with regard to its excellent learning capability with maximum 3-bit weight precision, which almost reaches the classification accuracy of the full-precision networks. The use of one-bit or two-bit shift weights can avoid multiplication operations that consume a considerable amount of power, so that it can significantly improve computational efficiency. Additionally, CBP is not a method tailored to particular models. Therefore, our work may have a broader impact on various application domains where memory capacity is limited and/or computational efficiency is of significant concern.

\section{Conclusion}
In this study, we proposed the CBP algorithm that trains DNNs by simultaneously considering both loss and constraint functions. It enables the implementation of any well-defined set of constraints on weights in a common training framework, unlike previous algorithms for weight quantization, which were tailored to particular constraints. The evaluation of CBP on ImageNet with with different constraint functions (binary, ternary, one-bit shift and two-bit shift weight constraints) demonstrated its high capability, highlighting its state-of-the-art accuracy of classification.

\begin{ack}
This work was supported by the Ministry of Trade, Industry \& Energy (grant no. 20012002) and Korea Semiconductor Research Consortium program for the development of future semiconductor devices and by National R\&D Program through the National Research Foundation of Korea (NRF) funded by Ministry of Science and ICT (2021M3F3A2A01037632).
\end{ack}

\bibliography{ref.bib}
\newpage
\setcounter{equation}{0}
\setcounter{figure}{0}
\setcounter{table}{0}
\setcounter{page}{1}
\appendix
\section{Appendix}
\subsection{Quantization kinetics in the continuous time domain}\label{kinetics_general}

The asymptotic quantization of weights $\bm{W}$ using BDMM with a Lagrangian function $\mathcal{L}$ follows the discrete updates,
\begin{eqnarray}\nonumber
    \bm{W} &\leftarrow& \bm{W} - \eta_W\nabla_{\bm{W}}\mathcal{L}\left(\bm{W},\bm{\lambda}\right)\\\nonumber
    \bm{\lambda} &\leftarrow& \bm{\lambda} + \eta_{\lambda}\nabla_{\bm{\lambda}}\mathcal{L}\left(\bm{x},\bm{\lambda}\right),\nonumber
\end{eqnarray}
which can be expressed in the continuous time domain as follows.
\begin{equation}\label{bdmm1cont}
    \dfrac{d\bm{W}}{dt}=-\tau_{W}^{-1}\nabla_{\bm{W}}\mathcal{L},
\end{equation}
and
\begin{equation}\label{bdmm2cont}
    \dfrac{d\bm{\lambda}}{dt}=\tau_{\lambda}^{-1}\nabla_{\bm{\lambda}}\mathcal{L},
\end{equation}
where the reciprocal time constants $\tau_{W}^{-1}$ and $\tau_{\lambda}^{-1}$ are proportional to learning rates $\eta_W$ and $\eta_{\lambda}$, respectively. The Lagrangian function $\mathcal{L}$ is a Lyapunov function of $\bm{W}$ and $\bm{\lambda}$.
\begin{equation}\label{lyapunov1}
    \dfrac{d\mathcal{L}}{dt} = \nabla_{\bm{W}}\mathcal{L}\cdot \frac{d\bm{W}}{dt} + \nabla_{\bm{\lambda}}\mathcal{L}\cdot \frac{d\bm{\lambda}}{dt}.
\end{equation}
Plugging Eqs. \eqref{bdmm1cont} and \eqref{bdmm2cont} into Eq. \eqref{lyapunov1} yields
\begin{equation}\label{lyapunov2}
    \dfrac{d\mathcal{L}}{dt} = -\tau_{W}^{-1}\left|\nabla_{\bm{W}}\mathcal{L}\right|^2 + \tau_{\lambda}^{-1}\left|\nabla_{\bm{\lambda}}\mathcal{L}\right|^2.
\end{equation}
The gradients in Eq. \eqref{lyapunov2} can be calculated from the Lagrangian function $\mathcal{L}$, given by
\begin{equation}\nonumber
    \mathcal{L} = C\left(\bm{y}^{(i)}, \hat{\bm{y}}^{(i)};\bm{W}\right)+\bm{\lambda}^{\rm T}\bm{cs}\left(\bm{W}\right),
\end{equation}
as follows.
\begin{eqnarray}\label{gradients}
    \left|\nabla_{\bm{W}}\mathcal{L}\right|^2 &=& \sum_{i=0}^{n_w}\left(\dfrac{\partial C}{\partial w_i}+\lambda_i\dfrac{\partial cs_i}{\partial w_i}\right)^2,\nonumber\\ 
    \left|\nabla_{\bm{\lambda}}\mathcal{L}\right|^2 &=& \sum_{i=0}^{n_w}cs_i^2. 
\end{eqnarray}
Therefore, the following equation holds.
\begin{equation}\label{app:lyapunov3}
    \dfrac{d\mathcal{L}}{dt} = -\tau_{W}^{-1}\sum_{i=0}^{n_w}\left(\dfrac{\partial C}{\partial w_i}+\lambda_i\dfrac{\partial cs_i}{\partial w_i}\right)^2 + \tau_{\lambda}^{-1}\sum_{i=0}^{n_w}cs_i^2.
\end{equation}
The Lagrange multiplier $\lambda_i$ at time $t$ is evaluated using Eq. \eqref{bdmm2cont}.
\begin{equation}\label{app:multiplier}
    \lambda_i\left(t\right)=\lambda_i\left(0\right)+\tau_{\lambda}^{-1}\int_0^t cs_i dt.
\end{equation}

\newpage
\subsection{Pseudocode}\label{app:pseudocode}
\begin{algorithm}
\SetAlgoLined
\KwResult{Updated weight matrix $\bm{W}$}
Pre-training using conventional backprop\;
Initialization such that $\bm{\lambda}\gets\bm{0}, p\gets 0, g\gets1$\;
Initial update of $\bm{\lambda}$\;
\For{$epoch = 1$ \KwTo $N$}{
    $\mathcal{L}_{sum}\gets$ $0$\;
    \tcc{Update of weight $\bm{W}$}
    \For{$i = 1$ \KwTo $M$}{
        $\bm{x}^{(i)}, \bm{\hat{y}}^{(i)}\gets$ minibatch$(\bm{Tr})$\;
        $\bm{y}^{(i)}\gets$ model$\left(\bm{x}^{(i)};\bm{W}\right)$\;
        $\mathcal{L}\gets$ \textit{C}$\left(\bm{\hat{y}}^{(i)}, \bm{y}^{(i)};\bm{W}\right)$ + $\bm{\lambda}^{\rm T}\bm{cs}\left(\bm{W};\bm{Q},\bm{M}, g\right)$\;
        $\mathcal{L}_{sum}\gets$ $\mathcal{L}_{sum}+\mathcal{L}$\;
        $\bm{W}\gets$ clip$\left(\bm{W}-\eta_W\nabla_{\bm{W}}\mathcal{L}\right)$\;}
    \tcc{Update of window variable $g$ and Lagrange multiplier $\bm{\lambda}$}
    $p\gets$ $p+1$\;
    \uIf{$\mathcal{L}_{sum}\geq$ $\mathcal{L}_{sum}^{pre}$ or $p=p_{max}$}{
        $g \gets$ $g + \Delta g$\;
        $\bm{\lambda}\gets \bm{\lambda}$ + $\eta_{\lambda}\bm{cs}\left(\bm{W}, g\right)$\;
        $p\gets$ $0$\;
        $\mathcal{L}_{sum}^{pre}\gets$ $\mathcal{L}_{sum}^{max}$\;}
    \Else{
        $\mathcal{L}_{sum}^{pre}\gets$ $\mathcal{L}_{sum}$\;}
    }
    \caption{CBP algorithm. $N$ denotes the number of training epochs in aggregate. $M$ denotes the number of mini-batches of the training set $\bm{Tr}$. The function $minibatch\left(\bm{Tr}\right)$ samples a mini-batch of training data and their targets from $\bm{Tr}$. The function $model\left(x, \bm{W}\right)$ returns the output from the network for a given mini-batch $\bm{x}$. The function clip($\bm{W}$) denotes the clipping weight, and $\eta_W$ and $\eta_{\lambda}$ denote the weight- and multiplier-learning rates, respectively.} 
    \label{algo}
\end{algorithm}

\subsection{Quantization kinetics with gradually vanishing unconstrained-weight window}\label{kinetics}
We consider the gradually vanishing unconstrained-weight window in addition to the kinetics of update of weights and lagrange multipliers in Eqs. \eqref{bdmm1cont} and \eqref{bdmm2cont}. Given that the update frequency of the unconstrained-weight window variable $g$ is equal to that of the Lagrange multipliers, its time constant equals $\tau_{\lambda}$.
\begin{equation}\label{bdmm3cont}
    \dfrac{dg}{dt}=\tau_{\lambda}^{-1}g_0,
\end{equation}
where $g_0 = 1$ when $g<10$, and $g_0=10$ otherwise. Regarding the Lagrangian function $\mathcal{L}$ as a Lyapunov function of $\bm{W}$, $\bm{\lambda}$, and $g$, Eq. \eqref{lyapunov1} should be modified as follow.
\begin{equation}\label{lyapunov_mod}
    \dfrac{d\mathcal{L}}{dt} = \nabla_{\bm{W}}\mathcal{L}\cdot \frac{d\bm{W}}{dt} + \nabla_{\bm{\lambda}}\mathcal{L}\cdot \frac{d\bm{\lambda}}{dt} + \dfrac{\partial\mathcal{L}}{\partial g}\dfrac{dg}{dt}.
\end{equation}
Plugging Eqs. \eqref{bdmm1cont}, \eqref{bdmm2cont}, and \eqref{bdmm3cont} into Eq. \eqref{lyapunov_mod} yields
\begin{equation}\label{lyapunov_mod2}
    \dfrac{d\mathcal{L}}{dt} = -\tau_{W}^{-1}\left|\nabla_{\bm{W}}\mathcal{L}\right|^2 + \tau_{\lambda}^{-1}\left|\nabla_{\bm{\lambda}}\mathcal{L}\right|^2 + \tau_{\lambda}^{-1} g_0 \dfrac{\partial\mathcal{L}}{\partial g}.
\end{equation}
The gradients in Eq. \eqref{lyapunov_mod2} can be calculated using Eqs. \eqref{lf}, \eqref{csf}, and \eqref{ucs1} as follows.
\begin{eqnarray}
    \left|\nabla_{\bm{W}}\mathcal{L}\right|^2 &=& \sum_{i=0}^{n_w}\left[\dfrac{\partial C}{\partial w_i}+\lambda_i\left(ucs_i\dfrac{\partial Y_i}{\partial w_i}+Y_i\dfrac{\partial ucs_i}{\partial w_i}\right)\right]^2,\\\label{gradients_mod}
    \left|\nabla_{\bm{\lambda}}\mathcal{L}\right|^2 &=& \sum_{i=0}^{n_w}\left(ucs_iY_i\right)^2,\nonumber\\
    \dfrac{\partial\mathcal{L}}{\partial g} &=& \dfrac{1}{2g^2}\sum_{i=0}^{n_w}\lambda_iY_i\sum_{j=1}^{n_q-1}\left(q_{j+1}-q_j\right)\delta\left(\dfrac{1}{2g}\left(q_{j+1}-q_j\right)-\left|w_i-m_j+\epsilon\right|\right).
\end{eqnarray}
Given that $\partial ucs_i/\partial w_i=0$ holds for any $w_i$ value because of $\epsilon\rightarrow 0^+$, $\left|\nabla_{\bm{W}}\mathcal{L}\right|^2$ is simplified as
\begin{equation}\label{gradient_lag}
    \left|\nabla_{\bm{W}}\mathcal{L}\right|^2=\sum_{i=0}^{n_w}\left(\dfrac{\partial C}{\partial w_i}+\lambda_iucs_i\dfrac{\partial Y_i}{\partial w_i}\right)^2.
\end{equation}
The gradient $\partial\mathcal{L}/\partial g$ is non-zero only if a given weight $w_i$ satisfies $\left|w_i-m_j+\epsilon\right|=\dfrac{1}{2g}\left(q_{j+1}-q_j\right)$ The probability that $w_i$ at a given time satisfies the equality for a given $g$ should be very low. Additionally, regarding the discrete change in $g$ in the actual application of the algorithm, the probability is negligible. Thus, this gradient can be ignored hereafter. Therefore, Eq. \eqref{lyapunov_mod2} can be re-expressed as
\begin{equation}\label{app:lyapunov_mod3}
    \dfrac{d\mathcal{L}}{dt} = -\tau_{W}^{-1}\sum_{i=0}^{n_w}\left(\dfrac{\partial C}{\partial w_i}+\lambda_iucs_i\dfrac{\partial Y_i}{\partial w_i}\right)^2 + \tau_{\lambda}^{-1}\sum_{i=0}^{n_w}\left(ucs_iY_i\right)^2.
\end{equation}

Distinguishing the weights belonging to the unconstrained-weight window $D_{ucs}$ from the others at a given time $t$, Eq. \eqref{app:lyapunov_mod3} can be written by
\begin{equation}\label{app:lyapunov_mod4}
    \dfrac{d\mathcal{L}}{dt} = -\tau_{W}^{-1}\sum_{w_i\in D_{ucs}}\left(\dfrac{\partial C}{\partial w_i}\right)^2 -\sum_{w_i\notin D_{ucs}}\left[\tau_{W}^{-1}\left(\dfrac{\partial C}{\partial w_i}+\lambda_i\dfrac{\partial Y_i}{\partial w_i}\right)^2 - \tau_{\lambda}^{-1}Y_i^2\right].
\end{equation}

\newpage
\subsection{Quantization kinetics in the discrete time domain}\label{kinetics_discrete}
\begin{figure}[tb]\centering
    \includegraphics[width=4.5in]{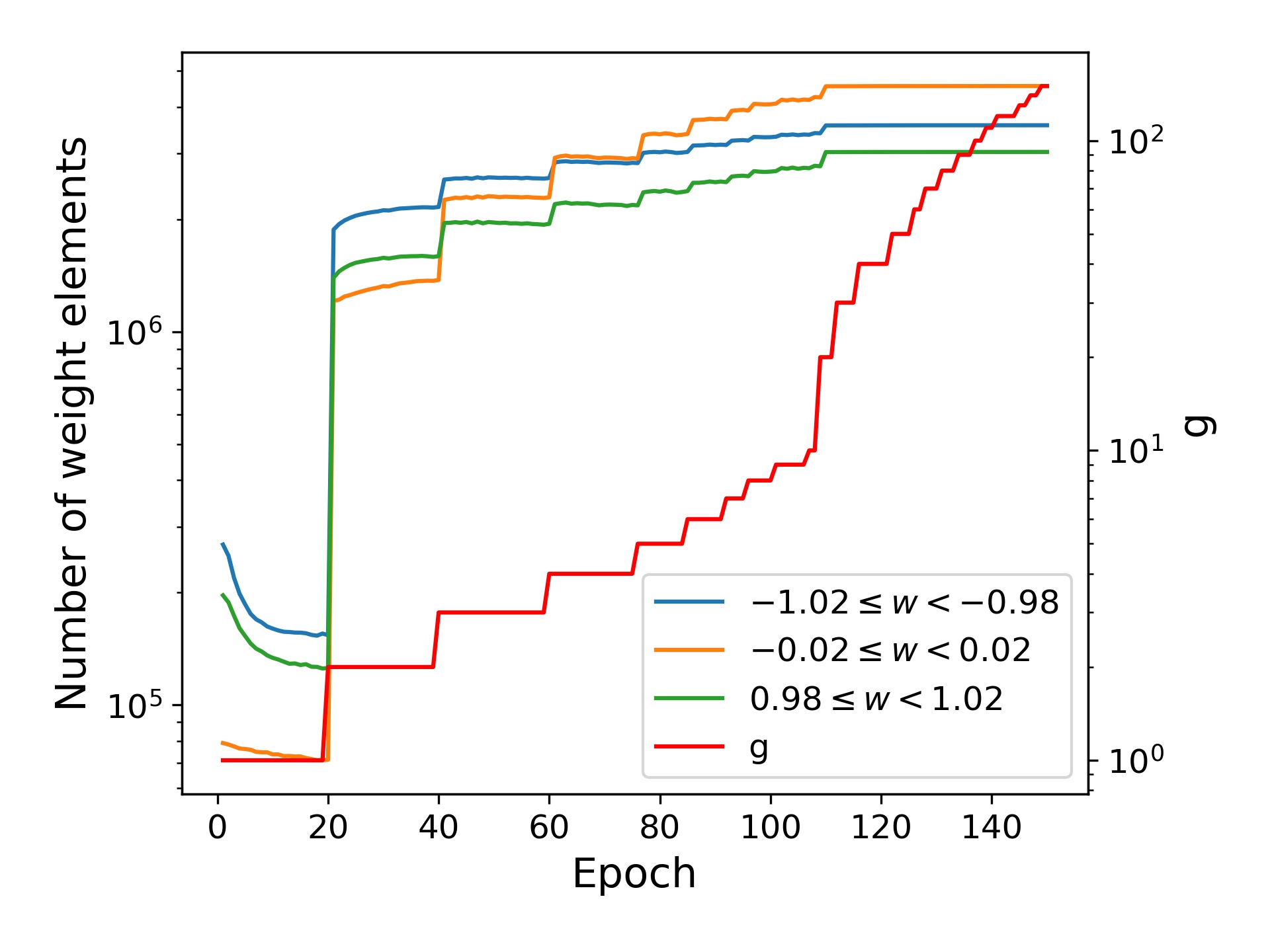}
    \caption{\label{learning_kinetics} Weight-ternarization kinetics of ResNet-18 on ImageNet}
\end{figure}
We monitored the population changes of weights near given quantized weight values for ResNet-18 on ImageNet with ternary-weight constraints. Fig. \ref{learning_kinetics} shows the population changes of weights near -1, 0, and 1 upon the update of the unconstrained-weight window variable $g$. As such, the variable $g$ was updated such that $\Delta g=1$ when $g<10$, and $\Delta g=10$ otherwise. Step-wise increases in populations upon the increase of $g$ are seen, indicating the obvious effect of the unconstrained-weight window on weight-quantization kinetics.  

\subsection{Hyperparameters}\label{app:hyperparameters}
The hyperparameters used are listed in Table~\ref{hyperparameter}. The weight- and multiplier-learning rates are denoted by $\eta_W$ and $\eta_{\lambda}$, respectively. The weight decay rate (L2 regularization) is denoted by $wd$.

\begin{table}[hbt!]
  \caption{Hyperparameters used.}
  \label{hyperparameter}
  \setlength{\tabcolsep}{1pt}
  \centering
  \begin{tabular}{@{\extracolsep{4pt}}ccccccccc}
    \toprule
     \multirow{2}{*}{}& \multicolumn{4}{c}{AlexNet} & \multicolumn{4}{c}{ResNet-18}\\
    \cline{2-5} \cline{6-9} 
     &$\eta_W$ & $\eta_\lambda$ & $wd$ & batch size &$\eta_W$ & $\eta_\lambda$ & $wd$ & batch size\\
    \midrule
    Binary & \multirow{2}{*}{$10^{-3}$}&\multirow{4}{*}{$10^{-4}$}& \multirow{4}{*}{$5\times10^{-4}$}&\multirow{4}{*}{$256$} & \multirow{4}{*}{$10^{-3}$}&\multirow{4}{*}{$10^{-4}$}& \multirow{4}{*}{$10^{-4}$}&\multirow{4}{*}{$256$}\\
    Ternary &  & & & \\
    One-bit shift & \multirow{2}{*}{$10^{-4}$}& & & \\
    Two-bit shift & & & & \\
    \midrule
    \multirow{2}{*}{}& \multicolumn{4}{c}{ResNet-50} & \multicolumn{4}{c}{GoogLeNet}\\
    \cline{2-5} \cline{6-9}
    &$\eta_W$ & $\eta_\lambda$ & $wd$ & batch size &$\eta_W$ & $\eta_\lambda$ & $wd$ & batch size\\
    \midrule
    Binary & \multirow{2}{*}{$10^{-3}$}&\multirow{4}{*}{$10^{-4}$}& \multirow{4}{*}{$10^{-4}$}&\multirow{4}{*}{$128$} &
    \multirow{4}{*}{$10^{-4}$}&\multirow{4}{*}{$10^{-4}$}& \multirow{4}{*}{$10^{-4}$}&\multirow{4}{*}{$256$}\\
    Ternary &  &  &  &  \\
    One-bit shift & \multirow{2}{*}{$10^{-4}$}& & &  \\
    Two-bit shift & & & & \\
    \bottomrule
  \end{tabular}
\end{table}

\subsection{Computational complexity}\label{app:complexity}
CBP is a post-training method so that this number of FLOPs is an additional computational complexity to the pre-training using backprop. 

\#FLOPs for CBP = (\#FLOPs for weight update) + (\#FLOPs for Lagrange multiplier update), where

\#FLOPs for weight update = (\#FLOPs for loss evaluation) + (\#FLOPs for error-backpropagation).

\#FLOPs for loss evaluation = (\#FLOPs for forward propagation) + (\#FLOPs for constraint contribution calculation $\bm{\lambda}^T\mathbf{cs}$). 

The number of FLOPs for the latter scales with the number of parameters in total ($n_w$) because each parameter is given a set of $\lambda$ and $cs$. The number of multiplication $\lambda\times cs_i\left(w_i\right)$ is the same as the number of parameters ($n_w$).The calculation of $cs_i$ for a given $w_i$ involves six FLOPs according to Eqs. (8)-(10). Therefore, 

\#FLOPs for loss evaluation = (\#FLOPs for forward propagation) + 6$n_w$.

As for conventional backprop, the number of FLOPs for weight update (using error-backpropagation) approximately equals the number of FLOPs for forward propagation. Therefore,

\#FLOPs for weight update = 2$\times$(\#FLOPs for forward propagation) + 6$n_w$

The Lagrange multiplier update for each multiplier involves one multiplication ($\eta_{\lambda}\times cs_i$) and one addition ($\lambda_i\gets \lambda_i+\eta_{\lambda}cs_i$), but uses $cs_i$ that has been calculated already when calculating the loss function. Therefore,

\#FLOPs for Lagrange multiplier update = 2$n_w$.

It should be noted that the multiplier is updated merely a few times during the entire training period: less than 20 percent of the training epochs, which is parameterized by $p$.

Therefore, we have

\#FLOPs for CBP = 2(\#FLOPs for forward propagation) + 2$(p+3)n_w$

The number of FLOPs for CBP for three models (for $p=0.2$) is shown below.

\textbf{AlexNet}: \#FLOPs for CBP $\approx$ 1.82G, and \#FLOPs for BP $\approx$ 1.45G (i.e., 25\% increase in \#FLOPs)

\textbf{ResNet18}: \#FLOPs for CBP $\approx$ 3.69G, and \#FLOPs for BP $\approx$ 3.62G (i.e., 2\% increase in \#FLOPs)

\textbf{ResNet50}: \#FLOPs for CBP $\approx$ 7.89G, and \#FLOPs for BP $\approx$ 7.74G (i.e., 2\% increase in \#FLOPs)

\section{Additional Data}\label{app:data}

\subsection{Extra Data}
Processes of learning quantized weights in AlexNet, ResNet-18, ResNet-50, and GoogLeNet are shown in Fig.~\ref{fig3}, \ref{fig4}, \ref{fig5}, and \ref{fig6}, respectively.

\begin{figure}[tb]
    \centering
    \includegraphics[width=5.3in]{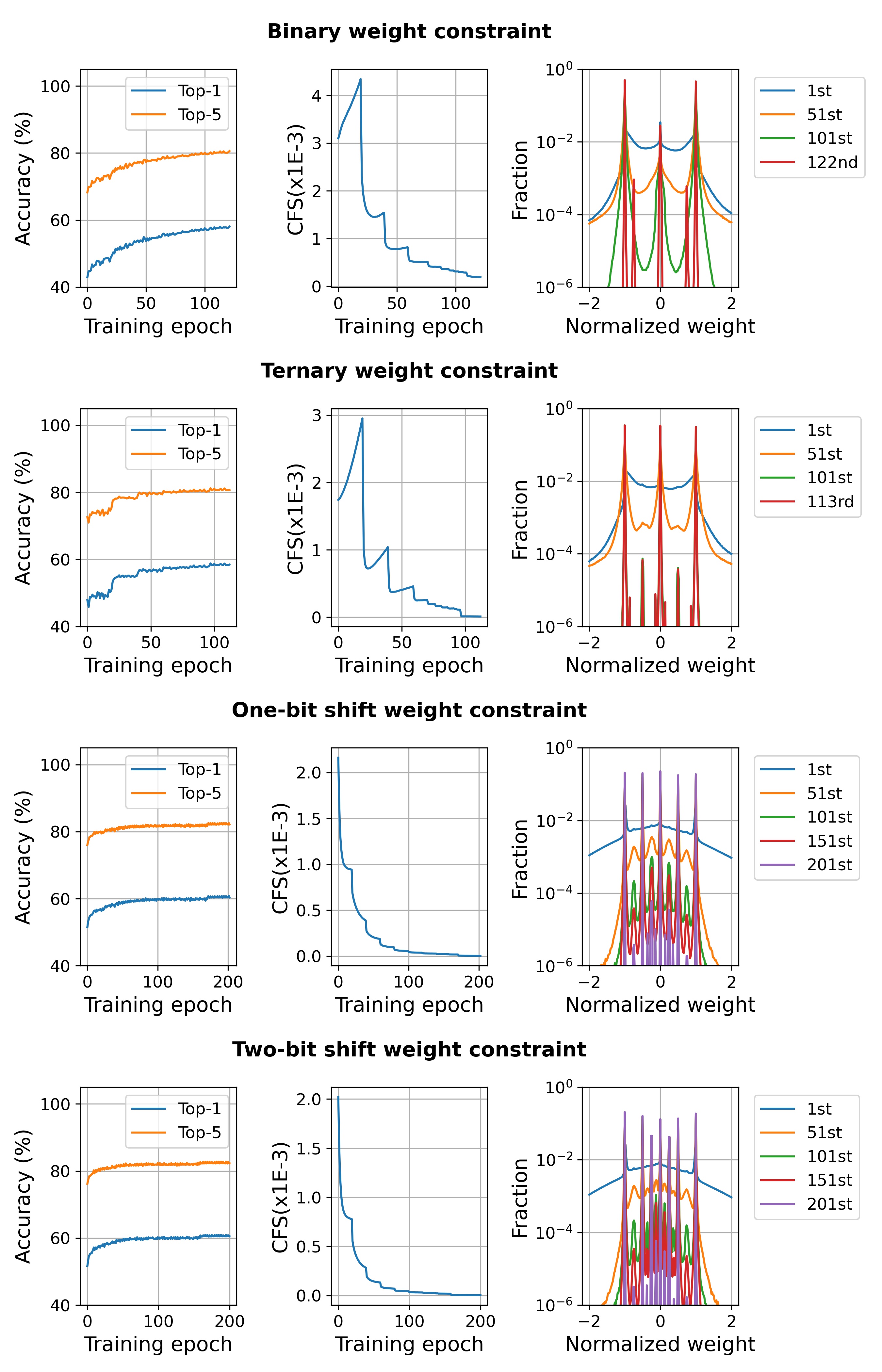}
    \caption{\label{fig3} Learning quantized weights in AlexNet}
\end{figure}
\begin{figure}[tb]
    \centering
    \includegraphics[width=5.3in]{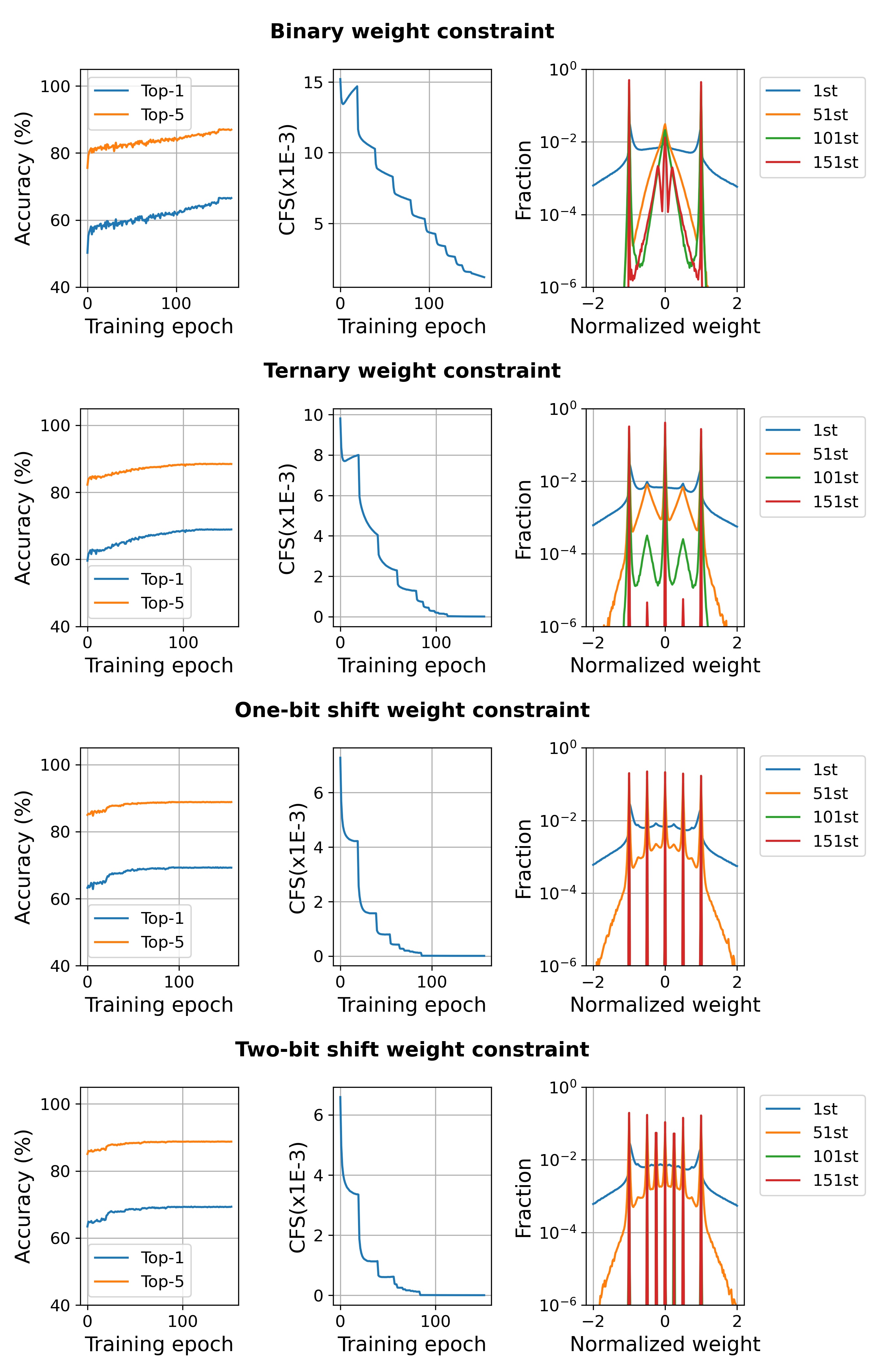}
    \caption{\label{fig4} Learning quantized weights in ResNet-18}
\end{figure}
\begin{figure}[tb]
    \centering
    \includegraphics[width=5.3in]{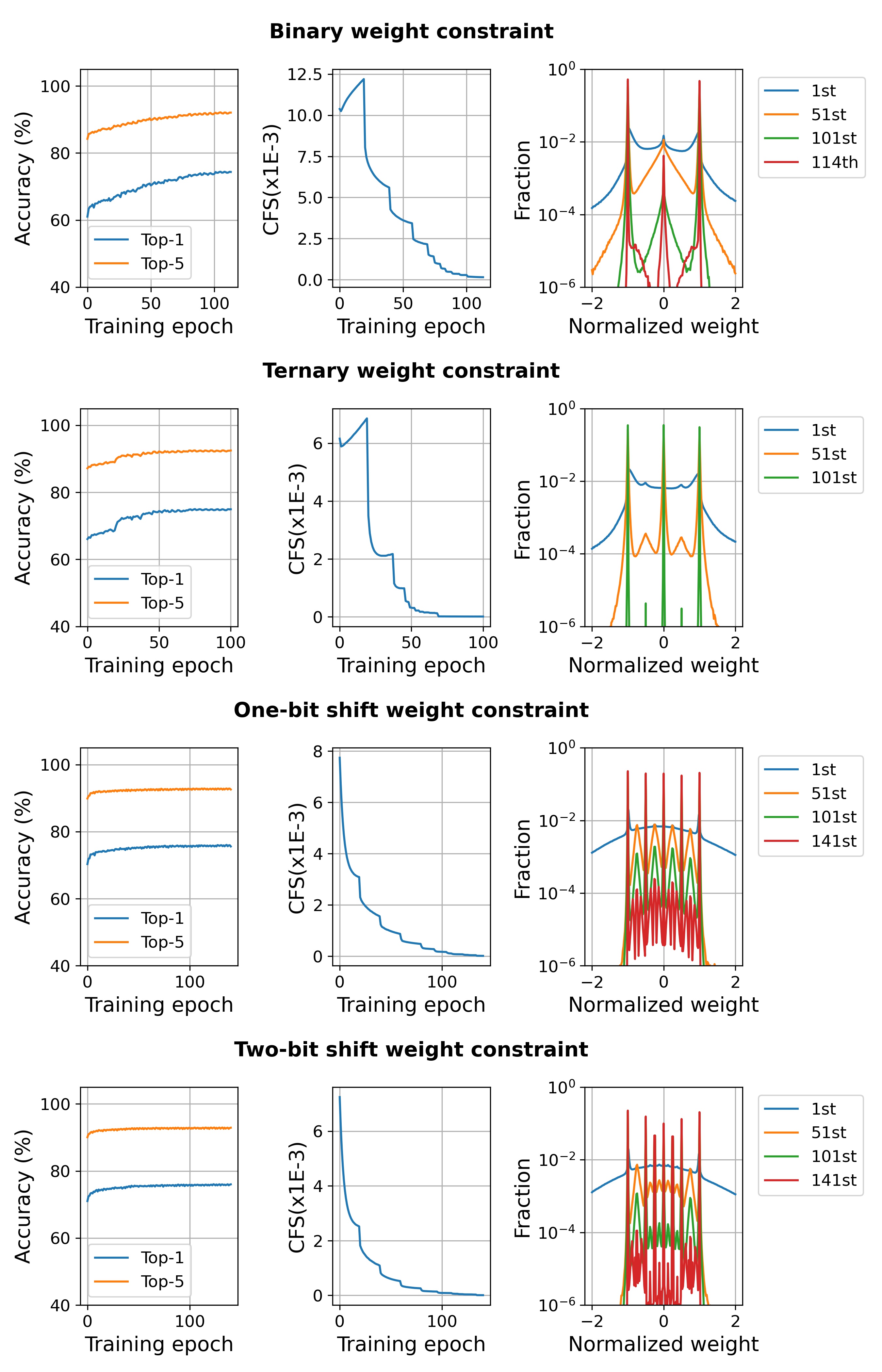}
    \caption{\label{fig5} Learning quantized weights in ResNet-50}
\end{figure}
\begin{figure}[tb]
    \centering
    \includegraphics[width=5.3in]{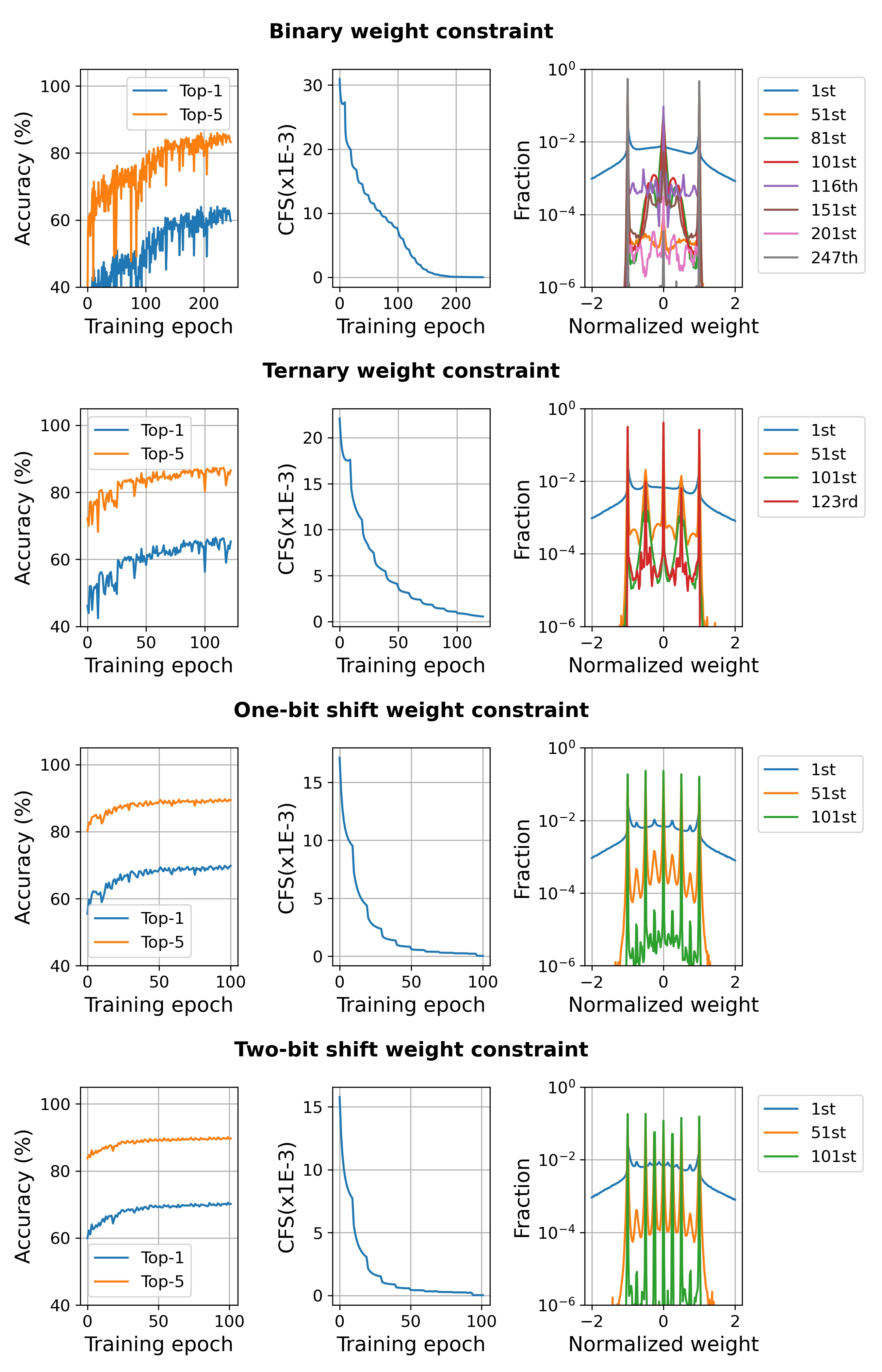}
    \caption{\label{fig6} Learning quantized weights in GoogLeNet}
\end{figure}
\end{document}